\documentclass{article}
\usepackage{kr}

\usepackage{times}
\usepackage{soul}
\usepackage{url}
\usepackage[hidelinks]{hyperref}
\usepackage[utf8]{inputenc}
\usepackage[small]{caption}
\usepackage{graphicx}
\usepackage{amsmath}
\usepackage{amsthm}
\usepackage[linesnumbered,ruled,vlined]{algorithm2e}
\usepackage{multirow} 
\usepackage{subfigure}
\usepackage{booktabs}
\usepackage{tikz}
\usepackage{bm}
\usepackage{amssymb}
\usepackage{algorithmic}
\usepackage{natbib} 
\newtheorem{example}{Example}

\newtheorem{definition}{Definition}
\newtheorem{corallary}{Corollary}
\newtheorem{proposition}{Proposition}

\title{When factorization meets argumentation: towards argumentative explanations}

\author{%
Jinfeng Zhong$^1$\and
Elsa Negre$^2$\\
\affiliations
$^1$Paris-Dauphine University, PSL Research
Universities, CEREMADE, 75016 Paris, France\\
$^2$Paris-Dauphine University, PSL Research
Universities, LAMSADE, 75016 Paris, France\\
\emails
jinfeng.zhong@dauphine.eu,
elsa.negre@lamsade.dauphine.fr
}

\begin{document}

\maketitle

\begin{abstract}
 Factorization-based models have gained popularity since the Netflix challenge {(2007)}. Since that, various factorization-based models have been developed and these models have been proven to be efficient in predicting users' ratings towards items. A major concern is that explaining the recommendations generated by such methods is non-trivial because the explicit meaning of the latent factors they learn are not always clear. In response, we propose a novel model that combines factorization-based methods with argumentation frameworks (AFs). The integration of AFs provides clear meaning at each stage of the model, enabling it to produce easily understandable explanations for its recommendations. In this model, for every user-item interaction, an AF is defined in which the features of items are considered as arguments, and the users' ratings towards these features determine the strength and polarity of these arguments. This perspective allows our model to treat feature attribution as a structured argumentation procedure, where each calculation is marked with explicit meaning, enhancing its inherent interpretability. Additionally, our framework seamlessly incorporates side information, such as user contexts, leading to more accurate predictions. We anticipate at least three practical applications for our model: creating explanation templates, providing interactive explanations, and generating contrastive explanations. Through testing on real-world datasets, we have found that our model, along with its variants, not only surpasses existing argumentation-based methods but also competes effectively with current context-free and context-aware methods.
\end{abstract}

\section{Introduction}


Factorization-based models have risen to prominence since the Netflix challenge in 2007 \citep{bennett2007netflix}. Matrix Factorization (MF) \citep{koren2009matrix} simplifies users' preferences into the dot product of user and item embeddings. A multitude of factorization-based models have emerged post-Netflix challenge \citep{rendle2010factorization, he2017neural, xiao2017attentional, he2017neural1}. However, the interpretations of the learned latent factors are often non-trivial, and the incorporation of Multi-Layer Perceptron (MLP) layers into factorization models adds to their opacity. This complexity complicates the process of discerning users' preferences for specific item attributes. Consequently, elucidating the rationale behind the recommendations made by these models is not straightforward. {Finding ``relevant-user'' is the common strategy applied to explain recommendations generated by factorization-based recommender techniques. Indeed, users usually have little knowledge (e.g. the items other users have interacted with before) about ``relevant-user'', which could hinder the trustworthiness of explanations \citep{zhang2018explainable}. At the same time, using other users' interactions may lead to privacy concerns.} A critical research question thus arises: How can we leverage the predictive strength of factorization-based models while ensuring the explainability of their recommendations?

In another line of research, argumentation framework (AF) has been applied to enhance the explainability of {Artificial Intelligence (AI)} \citep{vassiliades2021argumentation}. AF can graphically represent the decision-making process, the predefined properties help to reason how to reach the best decisions \citep{vassiliades2021argumentation}. AF also includes ways to define weighted arguments and dialectical relations between arguments, which contain meanings to decide how strong and acceptable these arguments are \citep{vcyras2021argumentative}. In the domains of RSs, there have been several works \citep {rago2018argumentation, rago2021argumentative, rago2020argumentation} that build explainable RSs through AF. However, the performances are limited and generating explanation require extracting subgraphs.  We will review these works in Section~\ref{sec:Argumentation_framework}.

In this work, we aim to build a factorization-based model that can provide model-intrinsic explanations. To this end, we propose Context-Aware Feature-Attribution Through Argumentation (CA-FATA). The key idea of CA-FATA is that users' ratings towards items depend on items' attributes. For example, in movie recommendation, movies' attributes mainly include actors, genres, and directors. The importance of attribute types varies across users. Some users may have a special preference for certain genres of movies, while some users may prefer movies starred by certain actors. Therefore it is important to capture this difference to accurately model users' preferences. CA-FATA computes the importance of different attribute types. The importance is further applied to compute users' ratings towards item attributes. Users' ratings towards different item attributes are weighted by the importance of attribute types and are aggregated to compute the final ratings towards items. We define an interaction-tailored AF where attributes of items are arguments and users' ratings towards attributes are the strength of these arguments, these arguments may support, attack or neutralize users' ratings towards items depending on users' ratings. This means that the computation of CA-FATA is endowed with explicit meaning, therefore, the prediction of ratings can be traced, which leads to model-intrinsic explanations. Moreover, our framework integrates auxiliary data, like user contexts \footnote[1]{Dey et al. \cite{abowd1999towards}: ``Context is any information that can be used to characterize the situation of an entity.''}. {The motivation is that users' preferences vary across contexts \cite{adomavicius2011context}. We believe that context is also crucial in AFs, as certain arguments that are considered ``good'' in one context may become less accurate in another context. Therefore, {it is important to leverage contexts when applying argumentation \citep{teze2018argumentative,garcia2004defeasible}}.} Further theoretical analyses indicate that the strength function of arguments under the AF for each user-item interaction satisfies two desired properties \emph{Weak balance} \citep{baroni2018many} and \emph{Weak monotonicity} \citep{baroni2019fine}, allowing to derive intuitive explanations in an argumentative way. We foresee three key applications for our framework: toy templates for explanations, offering interactive explanations, and producing contrastive explanations. Empirical evaluations using real-world datasets demonstrate that our model and its variants outperform existing argumentation-based approaches \citep{rago2018argumentation, rago2020argumentation} and are competitive with prevailing context-free and context-aware methods.

The following of this paper is organized as follows: Sections~\ref{sec:related} and~\ref{sec:Preliminary} introduce the key concepts in this paper; Section~\ref{sec:recommender} introduces the recommender model; Section~\ref{sec:argumentative_explanation} presents the argumentation scaffolding and explanation application; Section~\ref{sec:experiments} contains the experiments of empirical evaluation. Lastly, we conclude and propose perspectives.

\section{Related work}\label{sec:related}

{In this section, we will review the state-of-art related to this paper.}
\subsection{Factorization-based methods}

Factorization-based models, also known as latent factor approaches \citep{koren2021advances}, aim to discover underlying features that can be utilized to forecast ratings. Latent factor models comprise pLSA \citep{hofmann2004latent}, Latent Dirichlet Allocation \citep{blei2003latent}, and matrix factorization (MF) \citep{koren2009matrix}. Among these models, MF has gained popularity due to its attractive accuracy and scalability, particularly after the Netflix Prize competition \citep{bennett2007netflix}. In the following, we will compare several existing MF-based methods.

\textbf{Vanilla MF} \citep{koren2009matrix} is the inner product of vectors that represent users and items. Each user $u$ is represented by a vector $\textbf{p}_u \in \mathbb{R}^d$, each item $i$ is represented by a vector $\textbf{q}_i \in \mathbb{R}^d$, and $\hat{r}_{(u,i)}$ is computed by the inner product of $\textbf{p}_u $ and $\textbf{q}_i$: $\hat{r}_{(u,i)} = \textbf{p}_u{\textbf{q}_i}^T$

\textbf{Factorization machine (FM)} \citep{rendle2010factorization} takes into account user-item interactions and other features, such as users' contexts and items' attributes. It captures the second-order interactions of the vectors representing these features, thereby enriching FM's expressiveness. However, interactions involving less relevant features may introduce noise, as all interactions share the same weight \citep{xiao2017attentional}. 



\textbf{Neural Factorization machine (NFM)} \citep{he2017neural} is essentially derived from FM \citep{rendle2010factorization} and extends FM with a \emph{Bilinear Interaction}. The primary idea is that the \emph{Bilinear Interaction} encodes the second-order interactions of vectors representing features. Subsequently, a stack of fully connected layers is incorporated to capture the high-order interactions of vectors. 




\textbf{Attentional Factorization Machines (AFM)} \citep{xiao2017attentional} captures the second-order interactions of vectors representing users, items, and other features, similar to FM. Additionally, it models the attention of each interaction using a MLP. 


Factorization-based models contain a wide range of models and are a widely used approach in RSs, leveraging dimensionality reduction to identify latent features to explore observed user-item interactions. Explaining recommendations generated by factorization models is non-trivial. The latent features they identify do not have a clear, human-understandable meaning, which can make the recommendations they produce difficult to explain \citep{zhang2018explainable}. Often, researchers turn to post-hoc explanation techniques to explain recommendations \citep{zhong2022shap}. To explain directly the recommendations, we combine factorization with argumentation, in the next sub-section, we introduce the concepts related to argumentation.

\subsection{Argumentation frameworks}\label{sec:Argumentation_framework}

In recent years, argumentation-based methods have gained much attention in eXplainabll Artificial Intelligence (XAI) \citep{vassiliades2021argumentation, vcyras2021argumentative}. In fact, argumentation has been used to explain the decision-making process at the early stage of \emph{Argumentation Theory} \citep{pavese2019semantics}. This is because argumentation is intuitively understandable and explainable since the relations (e.g. attack, support.) in an AF can be easily understood. With the help of AFs, the decision-making process of AI systems can be mapped into an argumentation procedure, this helps to explain step by step how a decision has been achieved \citep{vassiliades2021argumentation}. AFs have been applied to build explainable systems in many domains, such as laws \citep{bench2020before}, where the authors discuss the application of AF in law; medical treatment \citep{sassoon2019explainable}, where the authors describe a wellness consultation system based on an AF. 

AFs have also been widely applied in RSs. Due to the natural explainability of argumentation, recommendations generated by argumentation-based methods are naturally explainable. \cite{toniolo2020dialogue} apply argumentation to build a decision-making system that can recommend medical treatment to patients suffering from multiple chronic health problems. Recommendations are justified through \emph{Satisfiability Modulo
Theories} \citep{barrett2018satisfiability} and patients receive explanations through argumentation dialogues. \cite{teze2015improving} apply Defeasible Logic Programming (DeLP) \citep{garcia2004defeasible} to generate recommendations adapted to users' contexts. \cite{briguez2014argument} propose a movie RS through a set of predefined recommendation rules. \citep{rago2018argumentation, rago2020argumentation} present an explainable recommendation framework called \emph{Aspect-Item framework (A-I)}. The framework computes users' ratings toward items in an argumentative way. More specifically, the framework can predict users' ratings toward the attributes of items (e.g. the actors of a movie; the genres of a movie, etc.), these ratings are further aggregated to compute users' ratings toward items. Our work is inspired by this series of works and is based on matrix factorization: we combine factorization and argumentation.

\section{Preliminary}\label{sec:Preliminary}

{In this section, we will lay out the key definitions that will be used in the following of this paper and also the formulation of problem. To standardize the terminology used in the rest of this paper, we present the key notions in Table~\ref{tab:notions}. Note that the bold font is used to represent the vectors that denote $cd$, $cs$, $cf$, $u$, $t$, and $at$.}

\subsection{Weak balance and weak monotonicity}
The concept of \emph{weak balance}, as defined by \cite{rago2018argumentation}, is a generalization of the notion of ``strict balance'' proposed by \cite{baroni2018many}. Similarly, the idea of \emph{weak monotonicity}, defined by \cite{baroni2019fine}, is a generalization of the concept of ``strict monotonicity'' as defined by \cite{rago2021argumentative}. The two properties allow for deriving intuitive explanations in an argumentative way. Essentially, the concept of \emph{Weak balance} concerns the impact of an argument on its affectees when the argument is the sole factor affecting them, while the idea of \emph{Weak monotonicity} focuses on how the potency of an argument changes when one of its affecters is silenced relative to the neutral point.


\textbf{Weak balance:}  Formally, \emph{weak balance} can be defined as follows:
\begin{definition} 
Given a Tripolar Argumentation Framework (TAF) $<\mathcal{A}, \mathcal{R}^{-}, \mathcal{R}^{+}, \mathcal{R}^{0}>$, $\sigma(at)$ satisfies the property of \emph{weak balance} if, for any $a,b \in \mathcal{A}:$
\begin{center}
$\bullet$ if  $\mathcal{R}^+(rec^i) = \{at\}$, $\mathcal{R}^-(rec^i) = \emptyset$ and $\mathcal{R}^0(rec^i) = \emptyset$ then $\sigma(rec^i) > 0$;
\\
$\bullet$ if  $\mathcal{R}^-(rec^i) = \{at\}$, $\mathcal{R}^+(rec^i) = \emptyset$ and $\mathcal{R}^0(rec^i) = \emptyset$ then $\sigma(rec^i) < 0$;
\\
$\bullet$ if  $\mathcal{R}^0(rec^i) = \{at\}$, $\mathcal{R}^+(rec^i) = \emptyset$ and $\mathcal{R}^-(rec^i) = \emptyset$ then $\sigma(rec^i) = 0$.
\end{center}
\end{definition}


\textbf{Weak monotonicity :}  Formally, \emph{weak monotonicity} is defined as follows:

\begin{definition} 
Given two TAFs $<\mathcal{A}, \mathcal{R}^{-}, \mathcal{R}^{+}, \mathcal{R}^{0}>$ and $<\mathcal{A}^{\prime}, \mathcal{R}^{{-}^{\prime}}, \mathcal{R}^{{+}^{\prime}}, \mathcal{R}^{{0}^{\prime}}>$, $(x,y) \in (\mathcal{R}^{-} \cup \mathcal{R}^{+} \cup \mathcal{R}^{0}) \cap (\mathcal{R}^{{-}^{\prime}}\cup\mathcal{R}^{{+}^{\prime}}\cup\mathcal{R}^{{0}^{\prime}})$. $\sigma$ satisfies \emph{weak monotonicity} at $(x,y)$ if, as long as $\sigma(x) = 0$ in $<\mathcal{A}^{\prime}, \mathcal{R}^{{-}^{\prime}}, \mathcal{R}^{{+}^{\prime}}, \mathcal{R}^{{0}^{\prime}}>$ and $\forall z \in [(\mathcal{R}^{-}(y) \cup \mathcal{R}^{+}(y)  \cup \mathcal{R}^{0})(y)  \cap (\mathcal{R}^{{-}^{\prime}}(y) \cup\mathcal{R}^{{+}^{\prime}}(y) \cup\mathcal{R}^{{0}^{\prime}})(y)]\backslash\{x\}$ and for $\sigma(z) = s$ in  $<\mathcal{A}, \mathcal{R}^{-}, \mathcal{R}^{+}, \mathcal{R}^{0}>$ and $\sigma(z) = s^{\prime}$ in $<\mathcal{A}^{\prime}, \mathcal{R}^{{-}^{\prime}}, \mathcal{R}^{{+}^{\prime}}, \mathcal{R}^{{0}^{\prime}}>$, $s = s^{\prime}$. Then the following holds for $\sigma(y) = v$ in  $<\mathcal{A}, \mathcal{R}^{-}, \mathcal{R}^{+}, \mathcal{R}^{0}>$ and $\sigma(y) = v^{\prime}$ in $<\mathcal{A}^{\prime}, \mathcal{R}^{{-}^{\prime}}, \mathcal{R}^{{+}^{\prime}}, \mathcal{R}^{{0}^{\prime}}>$
\begin{center}
$\bullet$ if $x \in \mathcal{R}^{-}(y) \cap \mathcal{R}^{{-}^{\prime}}(y)$, then $v^{\prime} > v $;
\\
$\bullet$ if $x \in \mathcal{R}^{+}(y) \cap \mathcal{R}^{{+}^{\prime}}(y)$, then $v^{\prime} < v $;
\\
$\bullet$ if $x \in \mathcal{R}^{0}(y) \cap \mathcal{R}^{{0}^{\prime}}(y)$, then $v^{\prime} = v $.
\end{center}
\end{definition}

\begin{table}[t]
\centering
\scriptsize
\caption{The key notions in this paper}
\begin{tabular}{@{}ll@{}}
\toprule
Notion                    & Description                                  \\ \midrule
$cf$ & A contextual factor                  \\
$C$&The contextual factors to characterize the situation of users\\
$cd$&A contextual condition\\
$\pi_u^{cf}$&The importance of contextual factor $cf$ to user $u$\\
$cs = (cd_1, cd_2, cd_3, \dots)$& A contextual situation\\
$\pi_u^t$&  The importance of type $t$ to user $u$  \\
$t_i$ & The feature types of item $i$
\\
$at_i$&The features of item $i$\\
\multirow{2}{*}{$at_i^{t}$} & The features of item $i$ of type $t$ (e.g. A movie may \\
&have several genres)\\
 $\mathcal{P}_u^{at}$ & The predicted rating of user $u$ towards feature $at$ \\
 $\mathcal{R^+}, \mathcal{R}^-, \mathcal{R}^0$ & Support, attack and neutral relations among arguments                  \\
$\mathcal{R}^-$(a,b) & Argument $a$ attacks argument $b$, the same for ``+'' and ``0''  \\
$\mathcal{A}$             & A set of arguments                           \\
\multirow{2}{*}{$rec^i$} &An argument stating \\
&``the item can be recommended to the target user''\\
$\mathcal{R}^+(rec^i) = \{at|(at,i) \in \mathcal{R}^+\}$&The arguments (features) that support $rec^i$\\
$\mathcal{R}^-(rec^i) = \{at|(at,i) \in \mathcal{R}^+\}$&The arguments (features) that attack $rec^i$\\
$\mathcal{R}^0(rec^i) = \{at|(at,i) \in \mathcal{R}^+\}$&The arguments (features) that neutralize $rec^i$\\
$<\mathcal{A}, \mathcal{R}^{-}, \mathcal{R}^{+}, \mathcal{R}^{0}>$& A tripolar argumentation framework\\
$\sigma(a)$               & The strength of argument $a$                 \\
 \bottomrule
\end{tabular}
\label{tab:notions}
\end{table}

\subsection{Problem formulation}\label{sec7.2.1}
 Considering the following recommendation scenario: for a target user $u$ under a contextual situation $cs$. The features of an item can have positive, negative, or neutral impacts on the prediction. In this regard, the features can be viewed as arguments, along with another argument stating ``the item can be recommended to the target user'', $rec^i$. Such a scenario can be seamlessly integrated into a TAF \citep{gabbay2016logical}. For a user-item interaction $(u,i)$ under $cs$, the TAF tailored to this interaction is a quadruple: $<\mathcal{A}, \mathcal{R}^{-}, \mathcal{R}^{+}, \mathcal{R}^{0}>$, where $\mathcal{A}$ contains $at_i$ (the features of item $i$) and $rec^i$. To illustrate this idea, consider the following example where $\mathcal{R}^{+}(at,rec^i)$ denotes that the feature $at$ has a positive effect on the recommendation of item $i$. Hence, the objectives are as follows: (i) to predict the rating assigned by a target user $u$ to a particular item $i$ in a given contextual situation $cs$; (ii) to determine the contribution of each item feature to a prediction under this contextual situation $cs$; (iii) to assess the polarity of each argument (feature) in the TAF; (iv) to design the strength function of arguments that comply with the conditions of \emph{weak balance} and \emph{weak monotonicity} as defined above.
\section{Recommender system}\label{sec:recommender}


The core ideas of Context-Aware Feature-Attribution Through Argumentation (CA-FATA) are as follows: (i) Users' ratings towards items depend on the features of items, which is also a fundamental concept of content-based RSs; (ii) The importance of feature types varies across users. For instance, when choosing a movie, some users may have specific preferences for certain actors while others may prefer movies of particular directors. In the former case, users pay more attention to the feature type \textbf{\emph{stars\_in}}, while in the latter case, they accord more importance to the feature type \textbf{\emph{directs}}; (iii) CA-FATA further extends this idea by noting that users' preferences also vary across contexts \citep{adomavicius2011context}. In the context of movie recommendations, some users are more likely to be influenced by \textbf{\emph{companion}} (e.g., with a lover or children) while others may be more likely to be influenced by \textbf{\emph{location}} (e.g., home or public place); (iv) We map item-features graphs into TAFs, where features of items are regarded as arguments that may attack, support, or neutralize the recommendation of the item; (v) The dialectical relations are learned in a data-driven way and supervised by users' past interactions with items (aka. ratings towards items).

\begin{figure}[t]
\centerline{\includegraphics[scale=0.20]{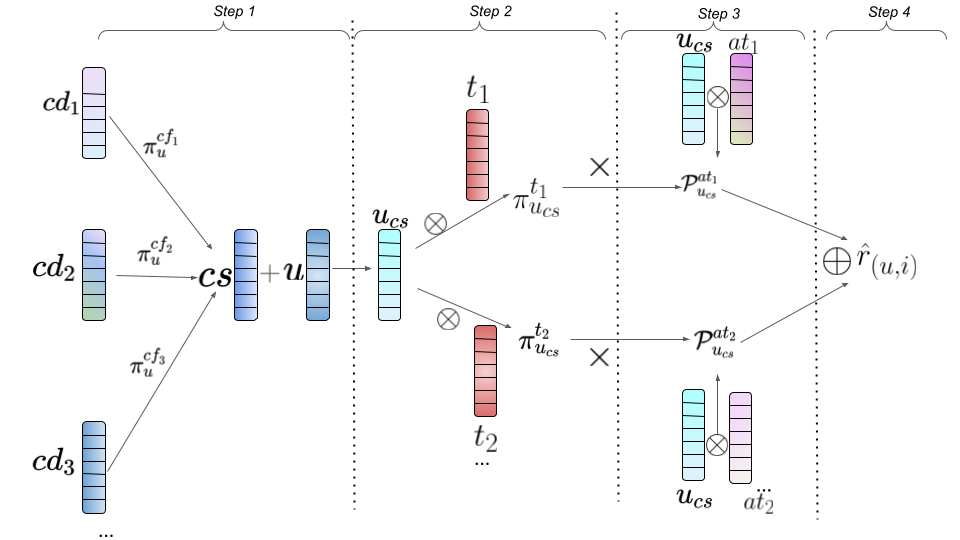}}
\caption[Illustration of the steps of CA-FATA.]{Illustration of the steps of CA-FATA. $\textbf{u}$, $\textbf{t}$, $\textbf{cf}$, \textbf{at} are vectors that represent user, feature type, contextual factor, and feature respectively; $\pi_u^{t}$ is the importance of feature type $t$ to user $u$; $\pi_u^{cf}$ is the importance of contextual factor $cf$ to user $u$; $\otimes$ denotes the inner product operation; $\times$ denote multiplication; $\oplus$ denotes addition.}
\label{fig:framework_fata}
\end{figure}

The structure of the CA-FATA comprises four successive steps. (0) First, we calculate the representation of target users within the given contextual situation. This initial step is crucial to ensure the tailoring of users' preferences and argument dialectical relations to the contextual circumstances. This step's objective is to procure a distinct user representation for each target contextual situation, ensuring that feature importance for each user is also adapted to the specific context. Additionally, it assures the polarity of the arguments (item features) aligns with the target contextual situation. (1) Next, we discern the significance of feature types relative to the aforementioned contextual situation. (2) The third step involves computing the users' ratings of item features within the contextual situation under consideration. This data then serves to delineate the dialectical relations. (3) The final step aggregates the ratings derived in the preceding step to generate users' overall item ratings.

\textbf{Computing user representation (Step 1):} As illustrated in Figure~\ref{fig:framework_fata}, the ultimate representation of a user is determined by the user's contextual situation. In this step, our goal is to compute the representation of users that is adapted to the target contextual situation. To achieve this, we begin by computing the score of the contextual factor $cf$ for a user $u$: $\beta_u^{cf} = g(\textbf{u},\textbf{cf})$, where $g$ is the inner product\footnote[2]{Note that other functions (e.g. MLP) may be adopted, but for simplicity, we adopt the inner product.}. After calculating the score $\beta_u^{cf}$ of all contextual factors to this user, {we normalize the score: $\pi_u^{cf} = \frac{exp(LeakyReLU(\beta_u^{cf}))}{\sum_{cf \in C}exp(LeakyReLU(\beta_u^{cf}))}$, as described by \cite{velivckovic2017graph}} to obtain the importance of each contextual factor. The importance computed here is similar to the relevance weight introduced by \cite{budan2020proximity}. However, unlike in these two works, where the relevance weight of the context is set empirically, in our work, the importance of the context is learned in a data-driven way. Intuitively, $\pi_u^{cf}$ characterizes the extent to which user $u$ wants to take contextual factor $cf$ into account.

In sequence, we compute the representation of the contextual situation $cs$ by summing up all the vectors representing contextual conditions multiplied by $\pi_u^{cf}$: $\textbf{cs} = \sum\limits_{cd \in cs} \pi_u^{cf}\textbf{cd}$, where $\textbf{cs}$ is the vector that denotes contextual situation $cs$. The next step is to aggregate the representation of contextual situation $cs$ with the representation of user $u$ to obtain a specific representation of user $u$ under the contextual situation $cs$. To avoid having an excessive number of parameters\footnote[3]{{Note that other aggregation methods such as concatenation are also possible but more parameters are induced. We leave this exploration for future work.}}, we sum $\textbf{u}$ and $\textbf{cs}$: $\bm{u_{cs}} =\textbf{u} + \textbf{cs}$. By aggregating information from a contextual situation $cs$ and a user $u$, each user $u$ gets a specific representation $\bm{u_{cs}}$ under a contextual situation $cs$.

\textbf{Computing feature type importance (Step 2):}

{The feature types in this paper are similar to the relations in knowledge graphs}, which are directed graphs consisting of \emph{entity-relation-entity} triplets \citep{hogan2021knowledge}. For instance, the triplet $(Harry Potter, hasDirector, Mike Newell)$ indicates that the movie \emph{Harry Potter} is directed by \emph{Mike Newell}. Here, $hasDirector$ is a relation in the knowledge graph that pertains to movies, and in this paper, it corresponds to the feature type $director$. We quantify the score between a feature type $t$ and a user $u$ as proposed by \cite{wang2019knowledge}: $\beta_{u_{cs}}^{t} = g(\bm{u_{cs}},\textbf{t})$, where $\bm{u_{cs}}$ is computed in the previous step, $\textbf{t}$ denotes the vector that represents feature type $t$. After calculating the score $\beta_{u_{cs}}^{t}$ of all feature importance to this user, {we normalize the score: $\pi_{u_{cs}}^t = \frac{exp(LeakyReLU(\beta_{u_{cs}}^{t}))}{\sum_{t \in t_i}exp(LeakyReLU(\beta_{u_{cs}}^{t}))}$

\textbf{Computing users' ratings towards features (Step 3):} 

To compute users' ratings towards features, we adopt the inner product again: $\mathcal{P}_{u_{cs}}^{at} = g(\bm{u_{cs}}, \textbf{at})$

According to Step 1 in Section~\ref{sec:recommender}, the representation of a user under one context differs from that under another context. As a result, the representation of user $u_{cs}$ is specific to each context, and the importance of feature type and user's rating towards features is also context-aware.

\textbf{Aggregating ratings towards features (Step 4):} After calculating the importance of each feature type and users' ratings towards each feature, the next step is to compute the rating of user $u$ towards item $i$ by computing the contribution of each feature type $t$ using the following equation:

\begin{equation} \label{equa:contribution_fata}
contr_t = \frac{\sum_{at \in at_i^t }\mathcal{P}_{u_{cs}}^{at}}{|at_i^t|}
\end{equation}
where $at_i^t$ denotes the features belong to type $t$ of item $i$. Finally, $u$'s rating towards $i$ under $cs$ is: 

\begin{equation} \label{equa:final_fata}
\hat{r}_{(u,i)} = \sum_{t\in t_i}\pi_{u_{cs}}^t * contr_t
\end{equation} 
where $t_i$ denotes all the feature types of item $i$. We would like to highlight that the actual user $u$'s rating for item $i$ is a real number ranging between $-1$ and $1$, in accordance with the definitions provided in earlier works \citep{rago2018argumentation,rago2021argumentative}. Furthermore, the striking resemblance between Equation~\ref{equa:final_fata} and the generalized additive model \citep{lou2012intelligible} is worth noting, signifying that our model belongs to the extended family of generalized additive models. This correlation facilitates straightforward identification of each feature's contribution, thereby achieving the first two objectives: \textbf{{calculating users' ratings for items and determining each feature's contribution defined in Section~\ref{sec7.2.1}}}.

\section{Explaining recommendations}\label{sec:argumentative_explanation}

In this section, we depict the details for explaining recommendations generated by CA-FATA.

\subsection{Argumentation scaffolding}\label{sec7.2.3}

In the following subsection, we aim to demonstrate how the computational steps outlined in Section~\ref{sec:recommender} can be accurately mapped onto an argumentation procedure. Our process begins by establishing a TAF that corresponds to each user-item interaction. This is accomplished by adhering to the steps outlined in CA-FATA. Subsequently, we delineate a strength function for arguments that aligns with both the principles of \emph{weak balance} \citep{rago2018argumentation} and \emph{weak monotonicity} \citep{rago2021argumentative}. This approach paves the way for determining the contribution of individual features.

\subsubsection{Argumentation setting}

CA-FATA first computes users' ratings towards features and then aggregates these ratings to predict ratings towards items. As a result, the features of items can be considered as arguments, and users' ratings towards features can be seen as the strength of these arguments. If CA-FATA predicts that a user's rating towards a feature of an item is high (under a target contextual situation $cs$), then this feature can be viewed as an argument that supports the recommendation of the item under $cs$. Conversely, if the predicted rating is low, the feature can be considered as an argument against the recommendation of the item under $cs$. In cases where features do not attack or support, a neutralizing relation is added, represented by $\mathcal{P}_{u_{cs}}^{at} = 0$. Therefore, we set $\sigma(at) = \mathcal{P}_{u_{cs}}^{at}$ and adopt TAF that contains support, attack, and neutralizing. Moreover, we set $\sigma(rec^i) =\hat{r}{(u,i)}$. When $\hat{r}{(u,i)} > 0$, the argument $rec^i$ is stronger if $\hat{r}_{(u,i)}$ is larger\footnote[4]{Note that if $\hat{r}_{(u,i)} < 0$ then the smaller $\hat{r}_{(u,i)}$ is, the stronger argument ``not recommending item $i$'' is.}. 

Recall that the true rating $r_{(u,i)}$ is a real number between -1 and 1, then the co-domain of $\mathcal{P}_{u_{cs}}^{at}$ is also expected to be between -1 and 1. Therefore, when $\mathcal{P}_{u_{cs}}^{at} > 0$, then $\sigma(at) > 0$, indicating that $at$ is an argument that supports $rec^i$\footnote[5]{Semantically, users $u$ prefers feature $at$.}; when $\mathcal{P}_{u_{cs}}^{at} = 0$, then $\sigma(at) = 0$, indicating that $at$ is an argument that neutralizes $rec^i$; when $\mathcal{P}_{u_{cs}}^{at} < 0$, then $\sigma(at) < 0$, indicating that $at$ is an argument that attacks $rec^i$. Therefore, the TAF corresponds to a user-item interaction $(u,i)$ under $cs$ can be defined as follows: 

\begin{definition}\label{def:taf_fata}
The TAF corresponding to $(u,i)$ under $cs$ is a 4-tuple: $<\mathcal{A}, \mathcal{R}^{-}, \mathcal{R}^{+}, \mathcal{R}^{0}>$ such that:
\begin{center}
$\bullet$ $\mathcal{R}^{-} = \{ (at,rec^i)|{\mathcal{P}_{u_{cs}}^{at}} < 0\}$;
\\
$\bullet$ $\mathcal{R}^{+} = \{(at,rec^i)|{\mathcal{P}_{u_{cs}}^{at}} > 0\}$;
\\
$\bullet$ $\mathcal{R}^{0} = \{(at,rec^i)|{\mathcal{P}_{u_{cs}}^{at}} = 0\}$.
\end{center}
\end{definition} 

According to Definition~\ref{def:taf_fata}, ${\mathcal{P}_{u_{cs}}^{at}}$ determines the polarity of arguments: if ${\mathcal{P}_{u_{cs}}^{at}}$ is positive then the argument (feature) supports the recommendation of item $i$ to user $u$; if ${\mathcal{P}_{u_{cs}}^{at}}$ is negative then the argument (feature) attacks the recommendation of item $i$ to user $u$; if ${\mathcal{P}_{u_{cs}}^{at}}$ is 0 then the argument neutralizes the recommendation. Therefore, the third goal: \textbf{{determining the polarity of each argument (feature)} under the target contextual situation defined in Section~\ref{sec7.2.1}}, is fulfilled.

\subsubsection{Proofs}

We will now show that by setting $\sigma(at) = \mathcal{P}_{u_{cs}}^{at}$ and $\sigma(rec^i) =\hat{r}{(u,i)}$,  TAF corresponding to $(u,i)$ under $cs$ satisfies \emph{weak balance} and \emph{weak monotonicity}. Recall that \emph{weak balance} states that attacks (or supports) can be characterized as links between affecters and affectees in a way such that if one affecter is isolated as the only argument that affects the affectee, then the former reduces (increases, resp.) the latter's predicted rating with respect to the neutral point.

\begin{proposition}\label{proposition:balance_fata}
Given the TAF corresponding to $(u,i)$ under $cs$, $\sigma(at) = \mathcal{P}_{u_{cs}}^{at}$ and $\sigma(i)$ satisfy \emph{weak balance}.
\end{proposition}

\begin{proof}
By inspecting Equations~\ref{equa:contribution_fata} and ~\ref{equa:final_fata}, it can be observed that since users' ratings are transformed into $[-1,1]$, the co-domain of $\sigma$ is also $[-1,1]$. Case (i): if $\mathcal{R}^+(rec^i) = \{at\}$, $\mathcal{R}^-(rec^i) = \emptyset$ and $\mathcal{R}^0(rec^i) = \emptyset$ then $\sigma(at) > 0$, therefore, $\mathcal{P}_{u_{cs}}^{at} > 0$, according to Equations~\ref{equa:contribution_fata} and ~\ref{equa:final_fata}, $\hat{r}{(u,i)} > 0 $, indicating that $\sigma(rec^i) > 0$. Case (ii): if $\mathcal{R}^-(rec^i) = \{at\}$, $\mathcal{R}^+(rec^i) = \emptyset$ and $\mathcal{R}^0(rec^i) = \emptyset$ then $\sigma(at) < 0$, therefore, $\mathcal{P}_{u_{cs}}^{at} < 0$, according to Equations~\ref{equa:contribution_fata} and ~\ref{equa:final_fata}, $\hat{r}{(u,i)} < 0 $, indicating that $\sigma(rec^i) < 0$. Case (iii): if $\mathcal{R}^0(rec^i) = \{at\}$, $\mathcal{R}^+(rec^i) = \emptyset$ and $\mathcal{R}^-(rec^i) = \emptyset$ then $\sigma(at) = 0$, therefore, $\mathcal{P}_{u_{cs}}^{at} = 0$, according to Equations~\ref{equa:contribution_fata} and ~\ref{equa:final_fata}, $\hat{r}{(u,i)} = 0 $, indicating that $\sigma(rec^i) = 0$. 
\end{proof}

Similar to  \emph{weak balance}, \emph{weak monotonicity} characterizes attacks, supports, and neutralizes as links between arguments such that if the strength of one affecter is muted then the strength of its affectees increases, decreases, and remains unchanged, respectively. Therefore, \emph{weak monotonicity} highlights the positive/negative/neutral effect between arguments. In this way, \emph{weak monotonicity} reveals the positive, negative, or neutral effect of an argument.

\begin{proposition}\label{proposition:monotonicity_fata}
Given the TAF corresponding to $(u,i)$ under $cs$, $\sigma(at) = \mathcal{P}_{u_{cs}}^{at}$ and $\sigma(i)$ satisfy \emph{weak monotonicity}.
\end{proposition}

\begin{proof}
\emph{Weak monotonicity} is formulated for two TAFs: from $<\mathcal{A}, \mathcal{R}^{-}, \mathcal{R}^{+}, \mathcal{R}^{0}>$ to $<\mathcal{A}^{\prime}, \mathcal{R}^{{-}^{\prime}}, \mathcal{R}^{{+}^{\prime}}, \mathcal{R}^{{0}^{\prime}}>$, after modifying certain arguments (e.g. muting certain features). Case (i): if $at \in \mathcal{R}^-(rec^i)$, then according to Definition~\ref{def:taf_fata}, ${\mathcal{P}_{u_{cs}}^{at}} < 0$, when $at$ is muted then ${\mathcal{P}_{u_{cs}}^{at}}^{\prime} = 0$. According to Equations~\ref{equa:contribution_fata} and ~\ref{equa:final_fata}, ${\hat{r}_{(u,i)}} ^{\prime} > \hat{r}_{(u,i)} $, indicating that ${\sigma(rec^i)}^{\prime} > \sigma(rec^i)$. Case (ii): if $at \in \mathcal{R}^+(rec^i)$, then according to Definition~\ref{def:taf_fata}, ${\mathcal{P}_{u_{cs}}^{at}} > 0$, when $at$ is muted then ${\mathcal{P}_{u_{cs}}^{at}}^{\prime} = 0$. According to Equations~\ref{equa:contribution_fata} and ~\ref{equa:final_fata}, ${\hat{r}_{(u,i)}} ^{\prime} < \hat{r}_{(u,i)} $, indicating that ${\sigma(rec^i)}^{\prime} < \sigma(rec^i)$. Case (iii): if $at \in \mathcal{R}^0(rec^i)$, then according to Definition~\ref{def:taf_fata}, ${\mathcal{P}_{u_{cs}}^{at}} = 0$, when $at$ is muted then ${\mathcal{P}_{u_{cs}}^{at}}^{\prime} = 0$. According to Equations~\ref{equa:contribution_fata} and ~\ref{equa:final_fata}, ${\hat{r}_{(u,i)}} ^{\prime} = \hat{r}_{(u,i)} $, indicating that ${\sigma(rec^i)}^{\prime} = \sigma(rec^i)$.
\end{proof}

Based on Proposition~\ref{proposition:monotonicity_fata}, the following corollary holds.
\begin{corallary}\label{corallary:2}
If the predicted rating of on $at \in at_i$ is increased: ${\mathcal{P}_{u_{cs}}^{at}}^{\prime} > \mathcal{P}_{u_{cs}}^{at}$, then the predicted rating of item $i$ also increases: ${\hat{r}_{(u,i)}} ^{\prime} > \hat{r}_{(u,i)}$. Accordingly, if the predicted rating one $at \in at_i$ is decreased: ${\mathcal{P}_{u_{cs}}^{at}}^{\prime} < \mathcal{P}_{u_{cs}}^{at}$, then the predicted rating of item $i$ also decreases: ${\hat{r}_{(u,i)}} ^{\prime} < \hat{r}_{(u,i)}$.
\end{corallary}

\begin{proof}
    By inspecting Equations~~\ref{equa:contribution_fata} and ~\ref{equa:final_fata}, Proposition~\ref{proposition:monotonicity_fata}, it is trivial to see that Corollary~\ref{corallary:2} holds.
\end{proof}

Propositions~\ref{proposition:balance_fata} and ~\ref{proposition:monotonicity_fata} guarantee that the fourth objective—\textbf{the formulation of a strength function that complies with \emph{weak balance} and \emph{weak monotonicity} defined in Section~\ref{sec7.2.1}}—is accomplished. In an argumentative fashion, the TAF related to CA-FATA elucidates how each feature sways the final rating prediction (recommendation): each feature may support, attack, or have no impact on the item's recommendation, with its polarity determined by the user's rating towards that feature. Therefore, CA-TAFA can be considered decomposable \citep{lipton2018mythos}, as each computation bears explicit meanings, which can be expressed in an intelligent manner \citep{lou2012intelligible}.

\subsection{Explanation applications}\label{sec:cafa_application}

\subsubsection{Toy explanation templates}\label{toy_dsaa}
 



Upon conducting the aforementioned analyses, we introduce three explanation templates: In every scenario, the most influential contextual condition ({computed in Step 1 in Section~\ref{sec:recommender}}) is selected. In the case of ``strong recommendation'', our approach proposes choosing the two most potent arguments (i.e., features) that support the recommendation for the item in question. Within the ``weak recommendation'' scenario, we propose the selection of the most influential argument in support of the item's recommendation, paired with the strongest argument against it. For situations where an item is ``not recommended'', our model suggests identifying the two most compelling arguments (i.e., features) that undermine the item's recommendation \footnote[6]{Please refer to \url{https://github.com/JinfengZh/ca-fata}.}. Each of these templates weaves together contextual information and corresponding arguments that either substantiate or contest the item's recommendation. To encapsulate, CA-FATA emerges as a versatile model capable of elucidating not only the reasons behind recommending certain items but also the rationale for discouraging others. In addition, CA-FATA offers the flexibility to create their own templates tailored to real needs.

\subsubsection{Interactive RSs}\label{interactive_dsaa}

\begin{figure}[t]
	\begin{center}
		
		\subfigure[CA-FATA recommends an item  and explains this recommendation while the user does not like the feature.]{\includegraphics[width=0.50\textwidth]{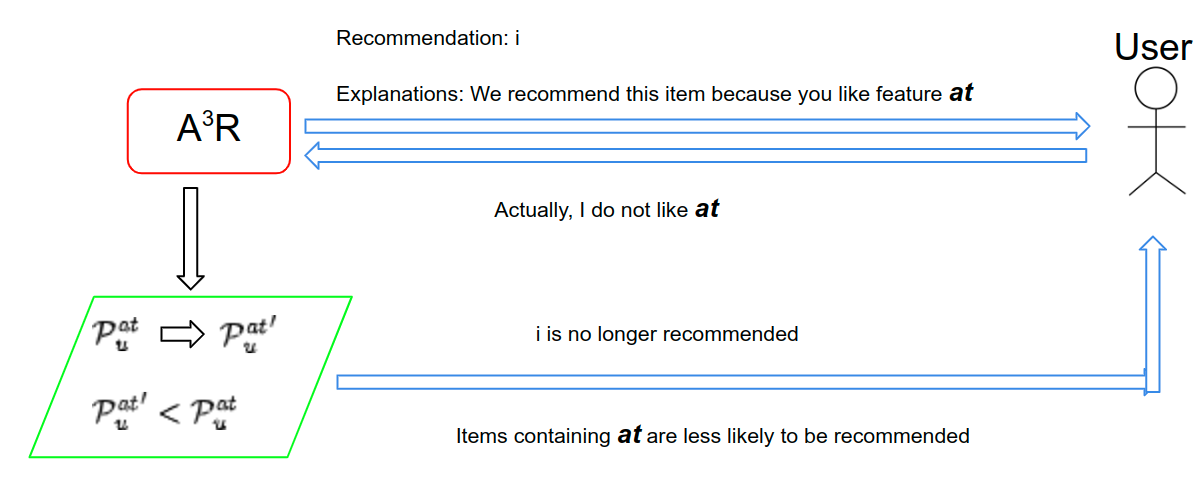}
			\label{computation_dsaa.sub.+}}
		\subfigure[CA-FATA does not recommend an item  and explains this behavior while the user likes the feature.]{\includegraphics[width=0.50\textwidth]{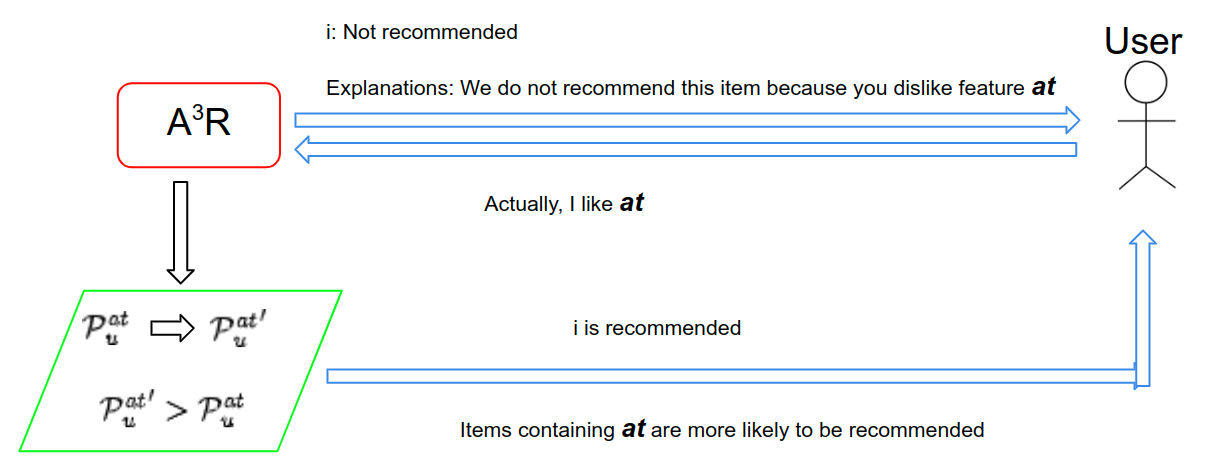}
			\label{computation_dsaa.sub.-}}
		\caption{Interaction process between users and CA-FATA}
		\label{fig:dsaa_interactive}
	\end{center}
\end{figure}

Another potential application of the CA-FATA model lies in its integration with interactive recommender systems (RSs), which take into account immediate user feedback to improve and adapt recommendations on the go \citep{he2016interactive, rago2020argumentation,rago2021argumentative}. This idea allows for a more dynamic and customized recommendation process, as compared to traditional, static RSs \citep{he2016interactive}. Following the optimization of CA-FATA parameters, the calculation of the importance of feature type and users' respective ratings towards each feature can proceed. When users receive a recommendation $i$ from the RS, it comes with the explanations generated by CA-FATA, such as ``We recommend this item because you like feature $at$.'' If users indicate a disinterest in feature $at$, the system can subsequently downgrade the rating of this argument. To illustrate, the system could modify $\mathcal{P}_u^{at}$ to ${\mathcal{P}_{u}^{at}}^{\prime}$, where ${\mathcal{P}_{u}^{at}}^{\prime} < \mathcal{P}_{u}^{at}$. Consequently, Corollary~\ref{corallary:2} ensures that ${\hat{r}{(u,i)}} ^{\prime} < \hat{r}{(u,i)}$, thus effectively preventing further recommendations of item $i$. {This process is illustrated in Figure~\ref{computation_dsaa.sub.+}.} Conversely, in scenarios where users appreciate item $i$ despite it not being recommended (refer to the ``not recommended'' scenario in Section~\ref{toy_dsaa}), the system explains, ``We do not recommend this item because you dislike feature $at$''. If users express that they do indeed favor feature $at$, the system can then increase the rating of this feature. The system could, for example, alter $\mathcal{P}_u^{at}$ to ${\mathcal{P}_{u}^{at}}^{\prime}$, under the condition that ${\mathcal{P}_{u}^{at}}^{\prime} > \mathcal{P}_{u}^{at}$. Consequently, according to Corollary~\ref{corallary:2}, ${\hat{r}{(u,i)}} ^{\prime} > \hat{r}{(u,i)}$. This adjustment ensures that item $i$ is recommended and items possessing feature $at$ are more likely to be favored in future recommendations, {which is illustrated in Figure~\ref{computation_dsaa.sub.-}.} This mechanism demonstrates that CA-FATA allows users to provide feedback when recommendations are unsuitable, thereby ensuring the scrutability of the explanations as articulated by \cite{tintarev2015explaining}.

\subsubsection{Contrastive explanations}\label{contrastive_dsaa}

\begin{algorithm}[t]\SetKwInOut{Input}{Input}\SetKwInOut{Output}{Output}
	\caption{Contrastive explanations} \label{alg:contrastive}

 \DontPrintSemicolon
 \SetAlgoLined

	\Input{CA-FATA: detailed in Section~\ref{sec:recommender}\\$u$: a target user\\$\mathcal{I} \backslash \mathcal{I}_u$: a set of items that $u$ has not interacted before}
	\Output{Explanations for recommending item} 
       $i_{rec} \leftarrow \mathop{\arg\max}\limits_{i \in \mathcal{I} \backslash \mathcal{I}_u} \hat{r}_{(u,i)}$, computed by CA-FATA\\
       $at_{pro} \leftarrow \mathop{\arg\max}\limits_{at \in at_{i_{rec}}}\pi_u^t*\mathcal{P}_u^{at}$\\
       $t_{pro} \leftarrow$ type of  $at_{pro}$\\
       $i_{con} \leftarrow \mathop{\arg\min}\limits_{i \in \mathcal{I}\backslash \{\mathcal{I}_u,i_{rec}\}}\hat{r}_{(u,i)}$\\
       $at_{con} \leftarrow \mathop{\arg\min}\limits_{at \in at^{t_{pro}}_{i_{con}} }\pi_u^t*\mathcal{P}_u^{at}$\\
        $Explanation \leftarrow$ ``We recommend $i_{rec}$ instead of $i_{con}$ because you prefer $at_{pro}$ and $i_{rec}$ is $at_{pro}$ while $i_{con}$ is $at_{pro}$.
\end{algorithm}
\vspace{-0,2cm}

As we detail in Section~\ref{sec:argumentative_explanation}, at the core of \emph{weak balance} and \emph{weak monotonicity} is the counterfactual reasoning to analyze the influence of arguments. Therefore, it would be natural to include contrastive explanations \citep{miller2019explanation}, which provide information not just about the features of the recommended item that led to it being recommended, but also how these features differ from those of other similar items that were not recommended (see the ``not recommended'' scenario in Section~\ref{toy_dsaa}).

Formally, Algorithm~\ref{alg:contrastive} illustrates the process of generating contrastive explanations using CA-FATA. Initially, the algorithm identifies a recommendation, denoted as $i_{rec}$ (line 1). Subsequently, it computes the feature contributing the most to $i_{rec}$'s recommendation (line 2), and identifies this feature type as $t_{pro}$ (line 3). The algorithm then identifies the lowest-rated item by the target user, denoted as $i_{con}$ (line 4). For the comparison to be logical, the algorithm specifies the feature of $i_{con}$ that belongs to $t_{pro}$ and receives the lowest rating from the target user (line 5). The recommendation for $i_{rec}$ can then be contrastively explained by comparing with $i_{con}$: ``We recommend $i_{rec}$ over $i_{con}$ because you have a preference for $at_{pro}$. $i_{rec}$ exhibits $at_{pro}$, whereas $i_{con}$ exhibits $at_{con}$.'' An example of a contrastive explanation for a movie recommendation could be:

\begin{example}\label{example:contrastive}
	We recommend Movie A over Movie B due to your preference for action movies. While Movie A fits this genre, Movie B is a drama.
\end{example}

In conclusion, CA-FATA is a highly adaptable model that provides explanations for both recommended and non-recommended items. Furthermore, it offers users the flexibility to define their own templates according to their specific requirements. We have also demonstrated that CA-FATA can be incorporated into interactive RSs to generate recommendations that align closely with users' preferences. Moreover, the model facilitates the generation of contrastive explanations by leveraging the inherent counterfactual reasoning within CA-FATA.

\subsection{Relations with other models}

{Having presented how our model can explain recommendations, we proceed to compare our model with that of A-I \citep{rago2018argumentation,rago2021argumentative}, which served as an inspiration for our work.} \cite{rago2018argumentation,rago2021argumentative} define users' profiles as ``collaborative filtering'' and ``type importance''. ``Collaborative filtering'' defines how users are influenced by ``similar users''. The ``similar users'' are defined as follows: $w_{u,v} = \frac{\textbf{u} \cdot \textbf{v}}{\|\textbf{u}\| \cdot  \|\textbf{v}\|}$ where $\textbf{u}$ and $\textbf{v}$ are vectors that represent users' preferences toward each genre (in their work, the genre of movies). ``type importance'' is the same as $\pi_u^t$ in our work, which defines how important each feature type is to users. A-I computes users' ratings toward each feature of an item and these ratings are regarded as the strength of arguments that lead to the final ratings of items, which makes recommendations generated by A-I explainable. We point out the traits of our CA-FATA compared with the A-I:




\begin{itemize}
    \item We do not explicitly explore users' ``similar users'', which is defined by users' preferences on genres of movies. We argue that this is too limited because users' preferences for actors and directors should be considered,  which are two other important feature types. 
    
    \item Unlike the A-I framework that directly sets the importance to all users (the importance of feature type to all users is the same), CA-FATA learns the importance for each user supervised by user-item interactions by leveraging the prediction power of Vanilla MF, which is, in fact, another form of collaborative filtering. In Section~\ref{sec:experiments}, we will show that CA-FATA can cluster users according to their preferences (aka, the importance of feature type to users).
    \item CA-FATA maps users' ratings towards items and features on an interaction-tailored AF, while in A-I \citep{rago2018argumentation,rago2020argumentation, rago2021argumentative} the authors map users' ratings towards items and features on a user-tailored AF, the graphical representation of a user-tailored AF could be rather large. Therefore, a sub-graph has to be extracted to explain recommendations. In CA-FATA, AFs are interaction-tailored. Therefore, all the arguments (features) are directly linked to items and can directly influence users' ratings towards items. As a result, there is no need for extracting sub-graphs of AFs to explain recommendations.  
\end{itemize}


\section{Empirical evaluations}\label{sec:experiments}

\begin{table}[t]
\scriptsize
\centering
\caption{Comparison between CA-FATA and baselines on RMSE and MAE, the second best are underlined. FATA is basically a variant of CA-FATA, the difference between CA-FATA and FATA is that FATA does not consider users' contexts and is actually the $A^3R$ \citep{zhong2022a3r} model. The version with ``AVG'' means that the importance of each feature type is set the same for all users.}
\label{tab:results_fata}
\begin{tabular}{ccllll}
\toprule
\multicolumn{2}{c}{\multirow{2}{*}{Model}} & \multicolumn{2}{c}{Yelp} & \multicolumn{2}{c}{Frapp\'e} \\
\cmidrule(l){3-6} 
\multicolumn{2}{c}{}  & RMSE & MAE & RMSE  & MAE \\ \toprule
\multirow{2}{*}{Context-free}  & MF  & 1.1809  & 0.9446 & 0.8761 & 0.6470 \\

 &NeuMF &1.1710&0.8815&0.6841 & 0.5207
 \\ \toprule
\multirow{5}{*}{Context-aware} & FM  & 1.1703 & 0.9412& 0.7067 & 0.5796\\
& CAMF-C & 1.1693  & 0.9241 & 0.7283 & 0.5727 \\
& LCM & 1.1687  & 0.9294 & 0.6952& 0.5396 \\
 & ECAM-NeuMF  & 1.1098  & \underline{0.8636} & 0.5599 &   0.4273\\
 \toprule
Argumentation-based  & A-I & 1.3978  & 1.1205& 1.1711 & 0.9848 \\\toprule
\multirow{4}{*}{Our propositions} & FATA & 1.1434 & 0.9059 & 0.6950 & 0.5439 \\
&AVG-FATA &1.1611&0.9314&0.6970&0.5461 \\
& \textbf{CA-FATA}  &\textbf{1.1033}&\textbf{0.8519} & \textbf{0.5154} & \textbf{0.3910} \\
&AVG-CA-FATA&\underline{1.1035}&0.8637&\underline{0.5254}&\underline{0.4025}\\

 \bottomrule
\end{tabular}

\end{table}

In the last section, we have shown how to explain recommendations generated by CA-FATA and the explanations applications. In this section, we conduct experiments on two real-world datasets to address the following research questions: (1) Can CA-FATA achieve competitive performance compared to baseline methods? What are the advantages of CA-FATA compared to baseline methods? (2) How does context influence the performance of CA-FATA? (3) How does the importance of feature type affect the performance of CA-FATA?
\subsection{Experiments setting}

We have conducted experiments on the following real-world datasets:

\textbf{Frapp\'e:} this dataset is collected by \cite{baltrunas2015frappe}. This dataset originated from Frapp\'e, a context-aware app recommender. There are 96 303 logs of usage from 957 users under different contextual situations, 4 082 apps are included in the dataset. Following \cite{unger2020context}, we apply log transformation to the number of interactions. As a result, the number of interactions is scaled to $0-4.46$. 

\textbf{Yelp\footnote[7]{\url{https://www.yelp.com/dataset}}:} this dataset contains users' reviews on bars and restaurants in metropolitan areas in the USA and Canada. Consistent with previous studies by \cite{zhou2020s3}, we use the records between January 1, 2019 to December 31, 2019, which contains 904 648 observations.

We compare the following baselines: (1) \textbf{Vanilla MF} \citep{koren2009matrix} (2) \textbf{CAMF-C} \citep{baltrunas2011matrix}; (3) \textbf{FM} \citep{rendle2010factorization}. (4) \textbf{NeuMF} \citep{he2017neural1} (5) \textbf{ECAM-NeuMF} \citep{unger2020context}: an extension of \textbf{NeuMF} that integrates contextual information. Note that \cite{unger2020context} do not release the implementation detail, for the NeuMF part, we empirically set the MLP factor size as 8, the sizes of the hidden layer as $(16,8,4)$, the GMF (Generalized Matrix Factorization) factor size as 16. This setting also applies to pure NeuMF. (6) \textbf{A-I} \citep{rago2018argumentation, rago2021argumentative}: an argumentation-based framework that computes users' ratings towards features, which are then aggregated to obtain the ratings towards items. Following \cite{rago2018argumentation}\footnote[8]{For detail please refer to \url{https://github.com/CLArg-group/KR2020-Aspect-Item-Recommender-System}.}, we set the ``collaborative factor'' as 0.8, 20 most similar users are selected, and all the feature importance is set as 0.1. 

\subsection{Experiment results}

Table~\ref{tab:results_fata} presents the results of the rating prediction experiments. We observe that CA-FATA performs better than all baselines on both the Yelp dataset and the Frapp\'e dataset, indicating its superiority in handling complex contextual information. The following are some specific observations: (i) CA-FATA performs well on the two datasets, outperforming all baselines, demonstrating its ability to model users' preferences under different contexts. Another advantage of CA-FATA is the ability to provide argumentative explanations, which is not possible for these baselines. (ii) Compared to A-I, on Yelp, CA-FATA achieves a significant reduction in RMSE and MAE. (iii) A horizontal comparison of Frappé and Yelp datasets shows that CA-FATA performs better on Frappé than on Yelp. We attribute this difference to the sparsity of the dataset, as Yelp is still highly sparse even after applying the 10-core setting, with a sparsity of $99.84\%$, while Frappé has a sparsity of $94.47\%$. To summarize, the advantages of CA-FATA are as follows: (i) it achieves competitive performance compared to both context-free and context-aware baselines. These baselines use factorization-based methods such as MF, FM, and CAMF-C, and some combine neural networks like NeuMF and ECAM-NeuMF, which makes them difficult to interpret. {On the other hand, CA-FATA provides explicit meanings for each computation and generates argumentative explanations (see Section~\ref{sec:argumentative_explanation} for some examples)}; (ii) compared to the argumentation-based method A-I, CA-FATA significantly improves prediction accuracy and generates context-aware explanations.

\subsection{Abalation study}\label{sec7.3.3}

\begin{figure*}[t]
\centering  
\subfigure[\footnotesize{Average importance of each contextual factor of users from cluster 0}]{
\label{Fig.sub.0_fata}
\includegraphics[width=0.20\textwidth]{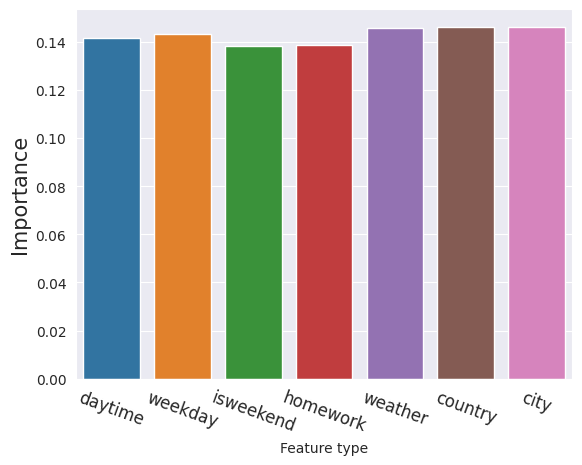}}
\subfigure[\footnotesize{Average importance of each contextual factor of users from cluster 1}]{
\label{Fig.sub.1_fata}
\includegraphics[width=0.20\textwidth]{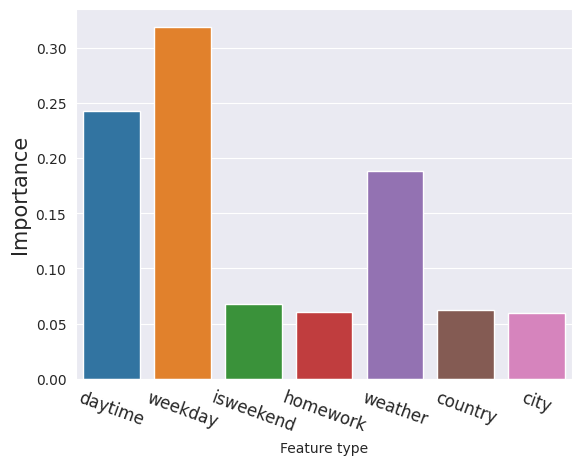}}
\subfigure[\footnotesize{Average importance of each contextual factor of users from cluster 2}]{
\label{Fig.sub.2_fata}
\includegraphics[width=0.20\textwidth]{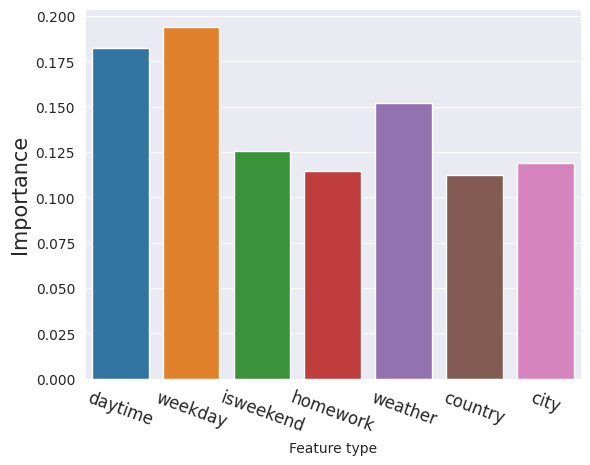}}
\subfigure[\footnotesize{Average importance of each contextual factor of users from cluster 3}]{
\label{Fig.sub.3_fata}
\includegraphics[width=0.20\textwidth]{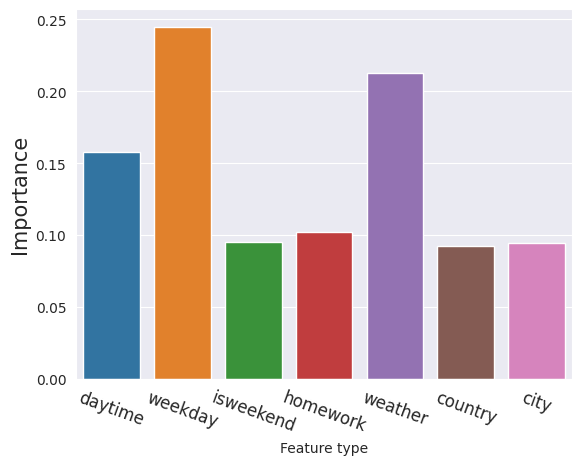}}
\caption[A case study on Frapp\'e dataset]{A case study on Frapp\'e dataset that shows the clustering of users according to the contextual factor importance learned by $CA-FATA$. The histogram shows the average importance of each contextual factor in the cluster.}
\label{Fig.main_fata}
\end{figure*}


In order to investigate the impact of contextual factors on the performance of CA-FATA, we propose an alternative approach called FATA, which neglects user contexts, identical to the $A^3R$ model proposed by \cite{zhong2022a3r}. Results presented in Table~\ref{tab:results_fata} (\textbf{refer to rows 10 and 12}) demonstrate that CA-FATA outperforms FATA, indicating that incorporating user contexts enables more nuanced modeling of user preferences and improves prediction accuracy. This conclusion is reinforced by the superior performance of CAMF-C over MF and ECAM-NeuMF over NeuMF. To investigate the influence of feature type importance on our proposed model's performance, we introduce AVG-CA-FATA and AVG-FATA for CA-FATA and FATA, respectively (\textbf{refer to rows 11 and 13 in Table~\ref{tab:results_fata}}). In these models, the importance of each feature type is uniformly set for all users. For instance, in Frappé, where there are five feature types, the importance is set to 0.2 for all users, while in Yelp, where there are three feature types, the importance is set to 0.33. Results demonstrate that AVG-CA-FATA performs worse than CA-FATA, as does AVG-FATA when compared to FATA. Furthermore, comparisons between FATA, AVG-FATA, CA-FATA, and AVG-CA-FATA confirm the advantages of incorporating user contexts and modeling feature type importance across users.

\subsection{Case study}\label{sec:case_cafata}

To further visualize the impact of context, we represent each user by their contextual factor importance. We use the Frappé dataset as an example, where a vector of seven dimensions represents each user. We therefore first apply K-means clustering for its simplicity and effectiveness \citep{velmurugan2010computational}, and find that four clusters best fit the dataset. We then use UMAP \citep{mcinnes2018umap} to visualize the clustering results. Note that other dimension reduction techniques (such as PCA and t-SNE \citep{van2008visualizing}, for the results of these methods please refer to \url{https://github.com/JinfengZh/ca-fata}) could also be used, {but we choose UMAP because it can preserve the underlying information and general structure of the data.} The average importance of each contextual factor for the users in the four clusters is shown in Figures~\ref{Fig.sub.0_fata}, ~\ref{Fig.sub.1_fata}, ~\ref{Fig.sub.2_fata}, ~\ref{Fig.sub.3_fata}, revealing that users in different clusters pay different levels of attention to contextual factors. Similar visualization applies to the Yelp data, due to limited space, we have omitted the visualization result. Please refer to \url{https://github.com/JinfengZh/ca-fata/blob/master/graph_yelp.ipynb}.

\section{Conclusions}\label{sec:conclusions}

In this paper, we propose a factorization-based model that can provide model-intrinsic explanations. To this end, we propose Context-Aware Feature-Attribution Through Argumentation (CA-FATA) that treats features as arguments that can either support, attack, or neutralize a prediction using argumentation procedures. This approach provides explicit meanings to each step and allows for easy incorporation of user context to generate context-aware recommendations and explanations. The argumentation scaffolding in CA-FATA is designed to satisfy two important properties: \emph{weak balance} and \emph{weak monotonicity}, {which highlights how features influence a prediction}. These properties help identify important features and study how they influence the prediction task. Additionally, our framework seamlessly incorporates side information, such as user contexts, leading to more accurate predictions. We anticipate at least three practical applications for our framework: creating explanation templates, providing interactive explanations, and generating contrastive explanations. Through testing on real-world datasets, we have found that our framework, along with its variants, not only surpasses existing argumentation-based methods but also competes effectively with current context-free and context-aware methods. 

CA-FATA follows the spirit of argumentation: features are regarded as arguments and users' ratings towards these arguments determine the polarity and the strength of them. In fully connected layers, each neuron performs a linear transformation on the input vector by utilizing a weights matrix. Regarding neurons as arguments may help to understand how neural networks achieve a prediction: the operations are mapped into AFs. There have been some primary works that integrate the spirit of argumentation into neural networks \citep{proietti2023roadmap, albini2020deep}. It would be interesting to explore how such integration benefits RSs. We also plan to conduct user studies to evaluate the three applications of the argumentative explanations presented in this paper.\newpage

\bibliographystyle{kr}
\bibliography{kr-sample}

\end{document}